\documentclass[twoside,11pt,preprint]{article}

\usepackage[abbrvbib]{jmlr2e}

\usepackage[utf8]{inputenc}
\usepackage{amsmath, mathrsfs, nicefrac, mathtools}
\usepackage[
    algo2e, 
    ruled, 
    linesnumbered, 
    noline, 
    noend, 
    ]{algorithm2e} 
\usepackage{pgfplots, pgfmath}
    \pgfplotsset{compat=1.14}
\usepackage{enumitem}
\usepackage{cleveref} 

\newcommand{\cF}{\mathcal{F}}

\newcommand{\cO}{\mathcal{O}}

\newcommand{\cR}{\mathcal{R}}

\newcommand{\cU}{\mathcal{U}}

\newcommand{\cX}{\mathcal{X}}

\newcommand{\E}{\mathbb{E}}

\newcommand{\I}{\mathbb{I}}

\newcommand{\N}{\mathbb{N}}

\newcommand{\Pb}{\mathbb{P}}

\newcommand{\R}{\mathbb{R}}

\newcommand{\Z}{\mathbb{Z}}

\newcommand{\fhi}{\varphi}

\newcommand{\lrb}[1]{\left(#1\right)}
\newcommand{\brb}[1]{\bigl(#1\bigr)}

\newcommand{\lsb}[1]{\left[#1\right]}
\newcommand{\bsb}[1]{\bigl[#1\bigr]}

\newcommand{\lcb}[1]{\left\{#1\right\}}
\newcommand{\bcb}[1]{\bigl\{#1\bigr\}}

\newcommand{\lce}[1]{\left\lceil#1\right\rceil}

\newcommand{\lfl}[1]{\left\lfloor#1\right\rfloor}
\newcommand{\bfl}[1]{\bigl\lfloor#1\bigr\rfloor}

\newcommand{\lno}[1]{\left\lVert#1\right\rVert}
\newcommand{\bno}[1]{\bigl\lVert#1\bigr\rVert}

\DeclareMathOperator*{\argmin}{argmin}

\newcommand{\dif}{\,\mathrm{d}}

\newcommand{\fracc}[2]{#1/#2}
\newcommand{\s}{\subset}

\newcommand{\iop}{\infty}
\newcommand{\bb}{\boldsymbol{B}}
\newcommand{\bB}{\boldsymbol{B}} 

\newcommand{\be}{\boldsymbol{e}}
\newcommand{\bp}{\boldsymbol{p}}
\newcommand{\bq}{\boldsymbol{q}}
\newcommand{\bw}{\boldsymbol{w}}
\newcommand{\bu}{\boldsymbol{u}}

\newcommand{\lcirc}{\ell^{\circ}}
\newcommand{\lcomp}{\lcirc}
\newcommand{\algoname}{Composite Loss Wrapper}
\newcommand{\uno}{(\mathrm{I})}
\newcommand{\due}{(\mathrm{II})}

\newcommand{\is}{i^\star}

\newcommand{\Lhat}{\widehat L}
\newcommand{\lhat}{\widehat \ell}
\newcommand{\loss}{\ell}
\newcommand{\colowr}{CoLoWrapper}

\newcommand{\Rlin}{R^{\mathrm{lin}}}
\newcommand{\bloss}{\boldsymbol{\ell}}
\newcommand{\cRlin}{\cR^{\mathrm{lin}}}
\renewcommand{\alpha}{A}
\DeclareSymbolFont{extraup}{U}{zavm}{m}{n}
\DeclareMathSymbol{\clubsuit}{\mathalpha}{extraup}{84}
\DeclareMathSymbol{\spadesuit}{\mathalpha}{extraup}{81}
\DeclareMathSymbol{\varheartsuit}{\mathalpha}{extraup}{86}
\DeclareMathSymbol{\vardiamondsuit}{\mathalpha}{extraup}{87}

\usepackage{lastpage}
\jmlrheading{23}{2022}{1-\pageref{LastPage}}{12/21}{8/22}{21-1443}{Nicol\`o Cesa-Bianchi, Tommaso Cesari, Roberto Colomboni, Claudio Gentile, Yishay Mansour}

\ShortHeadings{Nonstochastic Bandits with Composite Anonymous Feedback}{Cesa-Bianchi, Cesari, Colomboni, Gentile, Mansour}
\firstpageno{1}

\allowdisplaybreaks 

\begin{document}

\title{Nonstochastic Bandits with Composite Anonymous Feedback\thanks{A preliminary version of this paper appeared in the Proceedings of the 31st Conference on Learning Theory (COLT). PMLR 75:750-773, 2018.
}}

\author{\name Nicol\`o Cesa-Bianchi \email nicolo.cesa-bianchi@unimi.it \\
       \addr Dept.\ of Computer Science and DSRC, Universit\`a degli Studi di Milano, Italy
       \AND
       \name Tommaso Cesari \email tommaso.cesari@tse-fr.eu \\
       \addr Dept.\ of Computer Science and DSRC, Universit\`a degli Studi di Milano, Italy \\
       Institut de Math\'ematiques de Toulouse (IMT), Paul Sabatier University (UT3), Toulouse, France \\
       Toulouse School of Economics (TSE), Toulouse, France
       \AND
       \name Roberto Colomboni \email roberto.colomboni@unimi.it \\
       \addr Istituto Italiano di Tecnologia (IIT), Genova, Italy\\
       Universit\`a degli Studi di Milano, Italy
       \AND
       \name Claudio Gentile \email cgentile@google.com \\
       \addr Google Research, New York, USA
       \AND
       \name Yishay Mansour \email mansour@tau.ac.il \\
       \addr Tel Aviv University, Tel Aviv, Israel \\
       Google research, Tel Aviv, Israel
       }

\editor{Alexandre Proutiere}

\maketitle

\begin{abstract}%
We investigate a nonstochastic bandit setting in which the loss of an action is not immediately charged to the player, but rather spread over the subsequent rounds in an adversarial way.
The instantaneous loss observed by the player at the end of each round is then a sum of many loss components of previously played actions.
This setting encompasses as a special case the easier task of bandits with delayed feedback, a well-studied framework where the player observes the delayed losses individually.

Our first contribution is a general reduction transforming a standard bandit algorithm into one that can operate in the harder setting: We bound the regret of the transformed algorithm in terms of the stability and regret of the original algorithm. Then, we show that the transformation of a suitably tuned FTRL with Tsallis entropy has a regret of order $\sqrt{(d+1)KT}$, where $d$ is the maximum delay, $K$ is the number of arms, and $T$ is the time horizon.
Finally, we show that our results cannot be improved in general by exhibiting a matching (up to a log factor) lower bound on the regret of any algorithm operating in this setting.
\end{abstract}


\section{Introduction}\label{s:intro}
Multiarmed bandits, originally proposed for managing clinical trials, are now routinely applied to a variety of other tasks, including computational advertising, e-commerce, and beyond. Typical examples of e-commerce applications include content recommendation systems, like the recommendation of products to visitors of merchant websites and social media platforms. A common pattern in these applications is that the response elicited in a user by the recommendation system is typically not instantaneous, and might occur some time in the future, well after the recommendation was issued. This delay, which might depend on several unknown factors, implies that the reward obtained by the recommender at time $t$ can actually be seen as the combined effect of many previous recommendations to that user.

The more specific scenario of bandits with delayed rewards has been investigated in the literature under the assumption that the contributions of past recommendations to the combined reward is individually discernible ---see, e.g., \citep{neu2010online,joulani2013online,cgmm16,VernadeCP17}. 
\citet{pike2018bandits} revisited the problem of bandits with delayed feedback under the more realistic assumption that only the combined reward is available to the system, while the individual reward components remain unknown. This model captures a much broader range of practical settings where bandits are successfully deployed. Consider for example an advertising campaign which is spread across several channels simultaneously (e.g., radio, tv, web, social media). A well-known problem faced by the campaign manager is to disentangle the contribution of individual ads deployed in each channel from the overall change in sales.
\citet{pike2018bandits} formalized this harder delayed setting in a bandit framework with stochastic rewards, where they introduced the notion of \textsl{delayed anonymous feedback} to emphasize the fact that the reward received at any point in time is the sum of rewards of an unknown subset of past selected actions.
More specifically, choosing action $I_t \in [K]$ at time $t$ generates a stochastic reward $Y_t(I_t) \in [0,1]$ and a stochastic delay $\tau_t \in \{0,1,\dots\}$, where $\bcb{ Y_t(i), \tau_t }_{i\in [K],t\in \N}$ is a family of independent random variables such that $Y_1(i),Y_2(i),\dots$ have a common distribution $\nu_Y(i)$ (for all arms  $i\in[K]$) and $\tau_1,\tau_2,\dots$ have a common distribution $\nu_\tau$ with expectation $\mu_\tau$.
The delayed anonymous feedback assumption entails that the reward observed at time $t$ by the algorithm is the sum of $t$ components of the form $Y_s(I_s)\I\{\tau_s = t-s\}$ for $s\in \{1,\dots,t\}$.
The main result in \citep{pike2018bandits} is that, when the expected delay $\mu_{\tau}$ is known, the regret is at most of order of
$
    K\big((\ln T)/\Delta + \mu_{\tau}\big)
$,
where $\Delta$ is the suboptimality gap.
This bound is of the same order as the corresponding bound for the setting where the feedback is stochastically delayed, but not anonymous \citep{joulani2013online}, and cannot be improved in general.

In this work, we study a bandit setting similar to delayed anonymous feedback, but with two important differences. First, we work in a nonstochastic bandit setting, where rewards (or losses, in our case) are generated by some unspecified deterministic mechanism.
Second, we relax the assumption that the loss of an action is charged to the player at a single instant in the future. More precisely, we assume that the loss for choosing an action at time $t$ is adversarially spread over at most $d+1$ consecutive time steps $t,t+1,\dots,t+d$.
Hence, the loss observed by the player at time $t$ is a \textsl{composite loss}, that is, the sum of $(d+1)$-many loss components $\loss_{t}^{(0)}(I_{t}),\loss_{t-1}^{(1)}(I_{t-1}),\dots,\loss_{t-d}^{(d)}(I_{t-d})$, where $\loss_{t-s}^{(s)}(I_{t-s})$ defines the $s$-th loss component from the selection of action $I_{t-s}$ at time $t-s$. 
Note that in the special case when $\loss_t^{(s)}(i) = 0$ for all $s \neq d_t$, and $\loss_t^{(d_t)}(i) = \loss_t(i)$, we recover the model of nonstochastic bandits with delays $d_1, d_2,\dots \le d$ (which, in particular, reduces to the standard nonstochastic bandits when $d=0$). Our setting, which we call \textsl{composite anonymous feedback}, can accomodate scenarios where actions have a lasting effect which combines additively over time. Online businesses provide several use cases for this setting. For instance, an impression that results in an immediate clickthrough, later followed by a conversion, or a user that interacts with a recommended item ---such as media content--- multiple times over several days, or the free credit assigned to a user of a gambling platform which might not be used all at once.

Our main contribution is a general reduction technique (\algoname{}, or \colowr{}, Algorithm~\ref{a:wrapper}) turning a base nonstochastic bandit algorithm into one operating within the composite anonymous feedback setting. 
We then show that the regret of \colowr{} can be upper bounded in terms of the stability and the regret of the base algorithm (\Cref{t:regret-CoLoWr}).
Choosing as a base algorithm Follow the Regularized Leader (FTRL) with Tsallis entropy, \Cref{t:regret-CoLoWr} gives an upper bound of order $\sqrt{(d+1)KT}$ on the regret of nonstochastic bandits with composite anonymous feedback (Corollary~\ref{c:main-result}), where $d\ge 0$ is a known upper bound on the delay, $K$ is the number of actions, and $T$ is the time horizon. 
This result relies on a nontrivial stability analysis of FTRL with Tsallis entropy that could be of independent interest (\Cref{t:stability}).
Finally, we show the optimality of the $\sqrt{(d+1)KT}$ rate by proving a matching lower bound (up to a logarithmic factor, \Cref{t:lower}).
In particular, this shows that, in the nonstochastic case with delay $d$, anonymous feedback is strictly harder than nonanonymous feedback, whose minimax regret was characterized by \citet{cgmm16} as $\sqrt{(d+K)T}$. 
See the table below for a summary of results for nonstochastic $K$-armed bandits (all rates are optimal ignoring logarithmic factors).
\begin{center}
\begin{tabular}{|c|c|c|}
\hline
\textsc{no delay} & \textsc{delayed feedback} & \textsc{anonymous composite feedback}
\\ \hline\hline
\rule{0pt}{3ex} $\sqrt{KT}$ & $\sqrt{(d+K)T}$ & $\sqrt{(d+1)KT}$
\\
\citep{AuerCeFrSc02} & \citep{cgmm16} & (this paper)
\\ \hline
\end{tabular}
\end{center}
We now give an idea of the proof techniques.
Similar to \citep{pike2018bandits}, we play the same action for a block of at least $2d+1$ time steps, hence the feedback we get in the last $d+1$ steps contains only loss components pertaining to the same action, so that we can estimate in those steps the ``true loss'' of that action. Unfortunately, although the original losses are in $[0,1]$, the composite losses can be as large as $d+1$ (a composite loss sums $d+1$ loss components, and each component can be as large as $1$). 
This causes a corresponding scaling in the regret, compromising optimality. 
However, we observe that the total composite loss relative to the same action over any $d+1$ consecutive steps can be at most $2d+1$ (Lemma~\ref{l:lemmino}). 
Hence, we can normalize the total composite loss relative to the same action over $d+1$ consecutive steps, simply dividing by $2d+1$, obtaining an average loss in the range $[0,1]$. 
This idea leads to the right dependence on $d$ in the regret. 
The last problem is how to avoid suffering a big regret in the first $d$ steps of each block, where the composite losses mix loss components belonging to more than one action. 
We solve this issue by borrowing an idea of \citet{dhk14}.
We build blocks with random endpoints so that their length is (always) at least $2d+1$ and (on average) not much bigger.
This random positioning and length of the blocks is the key to prevent the oblivious adversary from causing a large regret in the first half of each block. 
Moreover, as we prove in \Cref{t:regret-CoLoWr}, if the distribution over actions maintained by the base algorithm is \textsl{stable} (Definition~\ref{d:stable}), then the algorithm is not significantly affected by the uncertainty in the positioning of the blocks.
Extending our results to the case where $d$ is unknown, \cite{wang2021adaptive} show a regret bound of order $T^{2/3}$.
When $d$ is known, however, their analysis does not guarantee our faster $\sqrt{T}$ rate.

\paragraph{Further related work}
Online learning with delayed feedback was studied in the full
information (non-bandit) setting by
\citet{weinberger2002delayed,mesterharm2005line,langford2009slow,joulani2013online,quanrud2015online,khashabi2016adversarial,joulani2016delay,garrabrant2016asymptotic,mann2019learning},
see also \citep{shamir2017online} for an interesting variant. The
bandit setting with delay was investigated in
\citep{neu2010online,joulani2013online,mandel2015queue,cgmm16,VernadeCP17,pike2018bandits,li2019bandit,thune2019nonstochastic,zimmert2020optimal,vernade2020linear,ito2020delay,gael2020stochastic,agrawal2020learning}.
Our delayed composite loss function generalizes the composite loss
function setting of \citet{ddkp14}---see the discussion at the end
of \Cref{s:prel} for details---and is also related to the
notion of loss functions with memory. This latter setting has been
investigated, e.g., by \cite{adt12}, who showed how to turn an
online algorithm with regret guarantee of $\cO(T^q)$ into one
attaining $\cO(T^{ 1/(2-q)})$-policy regret, also adopting a
blocking scheme. A more recent paper in this direction is
\citep{ahm15}, where the authors considered a more general loss
framework than ours, though with the benefit of counterfactual
feedback, in that the algorithm is aware of the loss it would incur
had it played any sequence of $d$ decisions in the previous $d$
rounds, thereby making their results incomparable to ours.

\section{Preliminaries}
\label{s:prel}

We denote the set 
of positive integers by $\N$ and the set of integers by $\Z$.
For all $n \in \N$ we denote the set $\{1,\dots,n\}$ of the first $n$ integers by $[n]$.
We will use the handy convention that, if $(c_t)_{t\in \Z} \s \R$ and $m,n\in\Z$ are such that $m>n$, then $\sum_{t=m}^n c_t = 0$ and $\prod_{t=m}^n c_t = 1$.
For any $x\in \R$, we denote its positive part $\max\{x, 0\}$ by $x^+$.

We start by considering a nonstochastic multiarmed bandit problem on
$K$ actions with oblivious losses in which the loss $\loss_t(i) \in
[0,1]$ at time $t$ of an action $i \in [K]$ is defined
by the sum
\[
  \loss_t(i) \coloneqq \sum_{s=0}^{d} \loss_t^{(s)}(i)
\]
of $(d+1)$-many components $\loss_t^{(s)}(i) \ge 0$ for $s \in\{0,\dots,d\}$.
Let $I_t$ denote the action chosen by the player at the
beginning of round $t$. If $I_t = i$, then the player incurs loss
$\loss_t^{(0)}(i)$ at time $t$, loss $\loss_t^{(1)}(i)$ at time $t+1$,
and so on until time $t+d$.
Yet, what the player observes at time $t$ is only the combined loss incurred at time $t$, which is the sum
\[
\loss_{t}^{(0)}(I_{t}) + \loss_{t-1}^{(1)}(I_{t-1}) + \cdots + \loss_{t-d}^{(d)}(I_{t-d})
\]
of the past $d+1$ loss contributions, where $\loss_t^{(s)}(i) = 0$ for all $i$ and $s$ when $t \le 0$. 
Then, we define the $d$-delayed \textsl{composite} loss at time $t$ of a sequence of $d+1$ actions $i_{t-d},\dots,i_t\in[K]$ as
\begin{equation}\label{e:mixedloss}
    \lcirc_t(i_{t-d},\dots,i_t)
\coloneqq
    \sum_{s=0}^d \ell_{t-s}^{(s)}(i_{t-s}) \;.
\end{equation}
With this notation, the $d$-delayed composite anonymous feedback assumption states that what the player observes at the end of each round $t$ is only the composite loss
$
\lcomp_t(I_{t-d},\dots,I_t)
$.
The goal of the algorithm is to bound its regret $R_T$ against the best fixed action in hindsight,
\[
    R_T 
\coloneqq
    \E \lsb{ \sum_{t=1}^T \lcirc_t(I_{t-d},\dots,I_t) } - \min_{i\in [K]} \sum_{t=1}^T \lcirc_t(i,\dots,i)~.
\]
We define the regret in terms of the composite losses $\lcomp_t$ rather than the true losses $\loss_t$ because in our model $\lcomp_t$ is what the algorithm pays overall on round $t$. It is easy to see that a bound on $R_T$ implies a bound on the more standard notion of regret $\E\left[\sum_{t=1}^T \loss_t(I_t)\right] - \min_{k}\sum_{t=1}^T \loss_t(k)$ up to an additive term of at most $\cO(d)$.

Our setting generalizes the composite loss function setting of \citet{ddkp14}.
Specifically, the linear composite loss function therein can be seen as a
special case of the composite loss~(\ref{e:mixedloss}) once we remove
the superscripts $s$ from the loss function components. In fact, in the linear case,
the feedback in \citep{ddkp14} allows one to easily reconstruct each individual
loss component in a recursive manner. This is clearly impossible in our more
involved scenario, where the new loss components that are observed in round $t$ need
not have occurred in past rounds.

\section{The \colowr{} Algorithm}
\label{s:wrapper}

Our Composite Loss Wrapper algorithm (Algorithm~\ref{a:wrapper}) takes as input a standard $K$-armed bandit algorithm $\alpha$ and a Boolean sequence $\bb$.
The base algorithm $\alpha$ operates on standard (noncomposite) losses with values in $[0,1]$,
producing probability distributions $\bq_1,\bq_2,\dots$ over the action set $[K]$. 
The wrapper calls the base algorithm $\alpha$ only in a subset of rounds determined by the Boolean sequence $\bb$, which we call update rounds.

\begin{definition}[Update round]
\label{d:update-rounds}
We say that $t\in \N$ is an \textsl{update round} with respect to a Boolean sequence $\bb = (b_t)_{t\in \N} \s \{0,1\}^{\N}$ if $t\ge 2d+1$ and $b_t \prod_{s=1}^{2d}(1-b_{t-s}) = 1$.
\end{definition}
Note that if $d>0$, the condition is equivalent to $b_t=1$, and $b_{t-1}=\ldots=b_{t-2d}=0$.
If $d=0$, by our convention, the condition is equivalent to $b_t=1$.

To help understand our algorithm, we will also define two other types of rounds.
We say that $t$ is a \textsl{draw round} if $t=1$ or the previous round $t-1$ was an update round.
If $t$ is not a draw round, we say that it is a \textsl{stay round}.
Note that, if $d=0$, both draw and stay rounds can be update rounds, while if $d\ge 1$, only stay rounds can be update rounds.
\begin{figure}
    \centering

    \begin{tikzpicture}[scale = 0.55]
    \foreach \x in {7.6, 17.6, 22.6} {\fill[yellow] ({\x},1.5) -- ({\x},-0.5) -- ({\x-1.2},-0.5) -- ({\x-1.2},0.5) -- ({\x-5.2},0.5) -- ({\x-5.2},1.5) -- cycle;}
    \draw[orange, thick] (7.6,1.5) -- (7.6,-0.5) (17.6,1.5) -- (17.6,-0.5);
        \foreach \x in {3.6, 12.6, 20.1} {\draw[orange] (\x, -0.5) node[below] {$\ge 2d+1$}; }
    \draw[orange, thick] (-0.4,-0.5) rectangle (22.6,1.5);
    \draw (-1,1) node[left] {$\bB$} (-1,0) node[left] {round};
    \def\Bern{{1,0,1,0,0,0,0,1,0,0,1,0,0,0,0,0,0,1,0,0,0,0,1}}
    \def\Roun{{"D","S","S","S","S","S","S","SU","D","S","S","S","S","S","S","S","S","SU","D","S","S","S","SU"}}
    \foreach \t in  {0,...,22}{
        \pgfmathparse{\Bern[\t]} 
        \node[label = center: $\pgfmathresult$] at (\t,1) {};
        \pgfmathparse{\Roun[\t]}
        \node[label = center: $\pgfmathresult$] at (\t,0) {};
    }
\end{tikzpicture}

    \caption{Sequence of rounds the algorithm is undergoing when $d=2$. 
    The top line contains the values of the Boolean sequence $\bb = (B_t)_{t\in \N}$.
    The bottom line shows the corresponding types of rounds: 
    each block begins with a ($D$)raw round,
    followed by a variable number of ($S$)tay rounds, the last of which is also an ($U$)pdate round.
    Since a round $t$ is an update round only if $B_t = 1$ and $B_s=0$ for the $2d$ previous rounds $s$, the length of each block is at least $2d+1$.
    }
    \label{f:blocks}
\end{figure}

The \colowr{} algorithm proceeds in blocks of (random) length of at least $2d+1$ rounds in which it constantly plays the same action (\Cref{f:blocks}).
Blocks in Algorithm \ref{a:wrapper} are counted by variable $n_t$.
Each block $n_t$ consists of a draw round followed by ($2d$ or more) 
stay rounds, with the last round of the block being also an update round.
During a draw round $t$, \colowr{} uses its current distribution $\bp_t$ to draw and play an action $I_t$.
During stay rounds, it keeps playing the action that was drawn during the latest draw round.
After playing the action $I_t$ for the current round $t$, if $t$ is an update round, \colowr{} asks the base algorithm $\alpha$ to make an update of its base distribution $\bq_{n_t} \to \bq_{n_t+1}$ as if $\alpha$ played action $I_t$ and observed as the loss of $I_t$ the quantity $\frac{1}{2d+1} \sum_{s = t-d}^t \lcirc_s (I_{s-d},\dots,I_s)$.
Then, the block ends and the distribution of \colowr{} at the beginning of the next block $n_{t+1}=n_t+1$ is $\bp_{t+1}=\bq_{n_t+1}$.

Note that if $t$ is an update round, the quantity $\frac{1}{2d+1} \sum_{s = t-d}^t \lcirc_s (I_{s-d},\dots,I_s)$ that is fed back to $\alpha$ relates only to the current action $I_t$, because blocks contain at least $2d+1$ rounds and the same action is played in all of them.
\begin{algorithm2e}[t]
\DontPrintSemicolon
\SetKwInput{KwIn}{input}
\SetKwInput{kwInit}{initialize}
\KwIn{Base $K$-armed bandit algorithm $A$ and Boolean sequence $\bb$}
\kwInit{let $n_0 = 0$ and $\bq_1$ be the initial distribution over $[K]$ of $\alpha$}
    \For
    {
        round $t=1,2,\dots$
    }
    {
        \If{either $t=1$ or $t-1$ was an update round (w.r.t.\ $\bb$)}
        {%
            let $n_t = n_{t-1}+1$, $\bp_t = \bq_{n_t}$, and draw $I_t \sim \bp_t$\tcp*{draw}
        }
        \Else{
            let $n_t = n_{t-1}$, $\bp_t=\bp_{t-1}$, and $I_t = I_{t-1}$\tcp*{stay}
        }
        play $I_t$ and observe loss $\lcirc_t(I_{t-d},\dots,I_t)$\;
        \If(\tcp*[f]{update}){$t$ is an update round (w.r.t.\ $\bb$)}
        {%
            feed $\alpha$ with arm $I_t$ and
            loss
            $
                \frac{1}{2d+1} \sum_{s = t-d}^t \lcirc_s (I_{s-d},\dots,I_s)
            $
            \label{s:update}
            \;
            use the update rule $\bq_{n_t} \to \bq_{n_t+1}$ of $\alpha$ to obtain a new base distribution $\bq_{n_t+1} \hspace{-1ex}$
        }
    }
\caption{\label{a:wrapper}\colowr{} (\algoname{})}
\end{algorithm2e}

The following lemma shows that this quantity is indeed in $[0,1]$, so that it is a legitimate feedback to pass to the base algorithm $\alpha$.

\begin{lemma}
\label{l:lemmino}
For all $t \ge 2d + 1$ and $i\in[K]$, 
\[
    \sum_{\tau = t - d}^t \lcirc_{\tau}(i,\dots,i) 
\le 
    2d + 1 \;.
\]
\end{lemma}
\begin{proof}
For all $t \ge 2d + 1$ and $i\in[K]$, 
\begin{multline*}
    \sum_{\tau = t - d}^t \lcirc_{\tau}(i,\dots,i)
=
    \sum_{\tau = t - d}^t \sum_{s=0}^{d}\ell^{(s)}_{\tau -s}(i)
=
    \sum_{s=0}^{d} \sum_{\tau = t - d}^t \ell^{(s)}_{\tau -s}(i)
= 
    \sum_{s=0}^{d} \sum_{\rho = t - d - s}^{t-s} \ell^{(s)}_{\rho}(i)
\\
\le
    \sum_{s=0}^{d} \sum_{\rho = t - 2d}^{t} \ell^{(s)}_{\rho}(i)
=
    \sum_{\rho = t - 2d}^{t} \sum_{s=0}^{d} \ell^{(s)}_{\rho}(i)
=
    \sum_{\rho = t - 2d}^{t} \ell_\rho(i)
\le
    t - (t - 2d) + 1
=
    2d + 1\;.
\end{multline*}
\end{proof}
As a final remark, we point out that, albeit the algorithm is parameterized with an entire sequence $\bb$, at each time $t$, it does not require the knowledge of the sequence at future times $t+1,t+2,\dots$. This implies in particular that these Boolean values could be produced and fed to \colowr{} in an on-line fashion.

\section{Upper Bound}
\label{s:upper}

We begin by formalizing the notion of stability (of the base algorithm), in terms of which we express the performance of the \colowr{} algorithm.

\begin{definition}[$\xi$-stability]
\label{d:stable}
Let $\xi>0$, $\alpha$ be a $K$-armed bandit algorithm, and $(\bq_n)_{n\in\N}$ be the (random) sequence of probability distributions over actions $[K]$ produced by $\alpha$ during a run over rounds $\{1,2,\dots\}$. We say that $\alpha$ is $\xi$-\textsl{stable} if for any round $n$, we have
\[
    \E \lsb{
        \sum_{i\in [K] } \brb{ \bq_{n+1}(i) - \bq_n(i) }^+ 
    } \le \xi
    \;.
\]
\end{definition}
In the previous definition, note that since $\sum_{i\in[K]} \bq_{n+1}(i) = 1 = \sum_{i\in[K]}  \bq_n(i) $, then
\[
    \bno{ \bq_{n+1} - \bq_n }_1
=
    \bno{ \bq_{n+1} - \bq_n }_1
    +
    \sum_{i\in[K]} \brb{ \bq_{n+1}(i) - \bq_n(i) }
=
    2 \sum_{i\in[K]} \brb{ \bq_{n+1}(i) - \bq_n(i) }^+ \;.
\]
Therefore, the $\xi$-stability of an algorithm is equivalent to controlling the expected $\lno \cdot _1$-distance between any two consecutive probability distributions produced by the algorithm.
We stick to the positive part definition as this is the quantity that naturally appears in the analysis.

We can now state our main result of this section.

\begin{theorem}
\label{t:regret-CoLoWr}
If we run \colowr{} with a $\xi$-stable base $K$-armed bandit algorithm $\alpha$ and an i.i.d.\ sequence $\bb = (B_t)_{t\in \N}$ of Bernoulli random variables with bias $\beta\in (0,1)$ (independent of the randomization of $\alpha$), then, for any time horizon $T \ge 2d + 1$, the regret $R_T$ satisfies
\[
    R_T
\le
    2d + \frac{2d+1}{d+1} \lrb{ 3d + 2d\beta(1-\beta)^{2d} \xi T + \frac{1}{\beta(1-\beta)^{2d}} \cR_{\lfl{ T/(2d+1) }} }
\]
where $\cR_{\lfl{ T/(2d+1) }}$ is the worst-case regret after $\bfl{ T/(2d+1) }$ rounds of $\alpha$ (for an adversarial setting with $[0,1]$-valued losses).
\end{theorem}

\begin{proof}
Fix an arbitrary horizon $T\ge 2d+1$ and an arm $\is\in [K]$. 
Let, for all\footnote
{%
Here, we refer to the infinite sequence of losses $(\ell_t^{(s)})_{s\in\{0,\dots,d\},t\in \N}$ (and the respective composite losses $(\lcirc_t)_{t\in \N}$). If the problem is formalized only with a finite sequence $(\ell_t^{(s)})_{s\in \{0,\dots,d\},t\in [T]}$, it is sufficient to define an arbitrary sequence $(\ell_t^{(s)})_{s\in\{0,\dots,1\},t>T}$ with $\ell_t^{(s)} \ge 0$ and $\sum_{s=0}^d\ell_t^{(s)} \le 1$ (and the corresponding composite losses $(\lcirc_t)_{t>T}$) and proceed as we do. This trick is needed to invoke Lemma~\ref{l:lemmino}, for which it is handy to sum $d$ rounds into the future.
}
$t\ge 2d+1$,
\[
    c_t 
\coloneqq 
    \E \bsb{  \lcirc_t(I_{t-d},\dots,I_t)  - \lcirc_t(\is,\dots,\is) }~,
\qquad
a \coloneqq 2d+1~, 
\qquad b \coloneqq T~. 
\]
Applying the elementary identity in Lemma~\ref{l:ragionieri} (Appendix~\ref{s:ragionieri}) in step $(*)$ below, we obtain
\begin{align*}
    R_T
& 
=
    \sum_{t=1}^{2d} c_t + \sum_{t=2d+1}^T c_t
\le
    \sum_{t=1}^{2d} \E\bsb{ \ell_t(I_t) } + \sum_{t=2d+1}^T c_t
\le
    2d + \sum_{t=2d+1}^T c_t
\\
&
=
    2d +
    \frac 1 {d+1} \lrb{ \sum_{t=a-d}^{a-1} (t-a+{d+1})\,c_t + (d+1) \sum_{t=a}^b c_t + \sum_{t=b+1}^{b+d} (b+{d+1}-t)\,c_t }
\\
& \qquad
    - \frac 1 {d+1} \sum_{t=a-d}^{a-1} (t-a+{d+1})\,c_t - \frac 1 {d+1} \sum_{t=b+1}^{b+d} (b+{d+1}-t)\,c_t
\\
& 
\overset{(*)}{=} 
    2d + \frac 1 {d+1} \sum_{\tau = a}^{b+d} \sum_{t = \tau-d}^\tau c_t 
    - \frac 1 {d+1} \sum_{t=a-d}^{a-1} (t-a+{d+1})c_t - \frac 1 {d+1} \sum_{t=b+1}^{b+d} (b+{d+1}-t)c_t
\\
& \le 
    2d + \frac 1 {d+1} \sum_{\tau = a}^{b+d} \sum_{t = \tau-d}^\tau c_t
    + \frac 1 {d+1} \sum_{t=a-d}^{a-1} (t-a+{d+1}) \lcirc_t(\is,\dots,\is)
\\
&
\qquad
    + \frac 1 {d+1} \sum_{t=b+1}^{b+d} (b+{d+1}-t) \lcirc_t(\is,\dots,\is)
\eqqcolon
    (\blacktriangledown)
\end{align*}
Now, applying Lemma~\ref{l:lemmino} in steps $(\circ)$ below, we get
\begin{align*}
    (\blacktriangledown)
& \overset{(\circ)}{\le} 
    2d + 2\frac{2d+1}{d+1}d + \frac 1 {d+1} \sum_{\tau = a}^{b+d} \sum_{t = \tau-d}^\tau c_t
\\
& =
    2d + 2\frac{2d+1}{d+1}d + \frac 1 {d+1} \sum_{\tau = 2d +1}^{T+d} \sum_{t=\tau-d}^\tau \E \bsb{ \lcirc_t(I_{t-d},\dots,I_t) - \lcirc_t(I_{\tau-2d},\dots,I_{\tau-2d}) }
\\
& \qquad
    + \frac{2d+1}{d+1} \E \lsb{ \sum_{\tau = 2d + 1}^{T+d} \frac 1 {2d+1} \sum_{t=\tau-d}^\tau  \brb{ \lcirc_t(I_{\tau-2d},\dots,I_{\tau-2d}) - \lcirc_t(\is,\dots,\is) } }
\\
& \overset{(\circ)}{\le}
    2d + 3\frac{2d+1}{d+1}d + \frac 1 {d+1} \sum_{\tau = 2d +1}^{T+d} \sum_{t=\tau-d}^\tau \E \bsb{ \lcirc_t(I_{t-d},\dots,I_t) - \lcirc_t(I_{\tau-2d},\dots,I_{\tau-2d}) }
\\
& \qquad
    + \frac{2d+1}{d+1} \E \lsb{ \sum_{\tau = 2d +1}^{T} \frac 1 {2d+1} \sum_{t=\tau-d}^\tau  \brb{ \lcirc_t(I_{\tau-2d},\dots,I_{\tau-2d}) - \lcirc_t(\is,\dots,\is) } }
\\
& \eqqcolon
    2d + 3\frac{2d+1}{d+1}d + \uno + \frac{2d+1}{d+1} \times \due~.
\end{align*}
We upper bound the two terms $\uno$ and $\due$ separately.
First, let $\cU$ be the (random) set of update rounds
\[
    \cU
\coloneqq
    \bcb{ \tau \in [T] : \tau \text{ is an update round (w.r.t.\ }\bb) } \;.
\]
For the first term $\uno$, we have
\begin{align*}
    \uno
& 
=
    \frac 1 {d+1} \sum_{\tau = 2d +1}^{T+d} \sum_{i\in [K]} 
    \E \lsb{ \sum_{t=\tau-d}^\tau \sum_{s=0}^{d} \ell_{t-s}^{(s)}(i)\brb{ \bp_{t-s}(i) - \bp_{\tau-2d}(i) } }
\\
& \le
    \frac 1 {d+1} \sum_{\tau = 2d + 1}^{T+d} \sum_{i\in[K]} \E\lsb{ \max_{t \in \{\tau -d, \dots, \tau\}, s \in \{0,\dots,d\}  } \lrb{ \bp_{t-s}(i) - \bp_{\tau-2d}(i) }} \sum_{t=\tau -d }^{\tau}  \lcirc_t(i,\dots,i)
\\
& \overset{(\circ)}{\le}
    \frac {2d+1} {d+1} \sum_{\tau = 2d + 1}^{T+d} \sum_{i\in[K]} \E\lsb{ \max_{t \in \{\tau -d, \dots, \tau\}, s \in \{0,\dots,d\}  } \lrb{ \bp_{t-s}(i) - \bp_{\tau-2d}(i) }}
\\
& \overset{(\varheartsuit)}{=}
    \frac {2d+1} {d+1} \sum_{\tau = 2d + 1}^{T+d} \sum_{i\in[K]} \E\lsb{ \I \lcb{ \bigcup_{\sigma = \tau -2d }^{\tau-1} \{\sigma \in \cU \} } \lrb{ \bp_{\tau}(i) - \bp_{\tau-2d}(i) }^+ }
\\
& =
    \frac {2d+1} {d+1} \sum_{\tau = 2d + 1}^{T+d} \sum_{i\in[K]} \sum_{\sigma=\tau-2d}^{\tau-1} \E\lsb{ \I\{\sigma \in \cU\} \lrb{ \bp_{\tau}(i) - \bp_{\tau-2d}(i) }^+  }
\eqqcolon
    (\vardiamondsuit)
\end{align*}
where $(\circ)$ follows by Lemma~\ref{l:lemmino} together with the fact that the $\max$ in the previous line is always nonnegative (to see this, simply observe that picking $t = \tau -d$ and $s=d$ within the max makes $\bp_{t-s}(i) = \bp_{\tau-2d}(i)$), and 
$(\varheartsuit)$ follows by the facts that the $\max$ in the previous line is always greater than or equal to zero, it can be strictly positive only if there is an update in a round $\sigma$ with $\tau-2d \le \sigma \le \tau-1$, and there can be at most a single update in $2d+1$ consecutive time steps.
Now, using the facts that $\bp_\tau = \bq_{n_\tau}$ and that $n_{\tau}=n_{\tau-2d}+1$ on the event $\{\sigma \in \cU\}$, we get
\begin{align*}
    (\vardiamondsuit)
& =
    \frac {2d+1} {d+1} \sum_{\tau = 2d + 1}^{T+d} \sum_{i\in[K]} \sum_{\sigma=\tau-2d}^{\tau-1}  \E\lsb{ \I\{\sigma \in \cU\} \lrb{ \bq_{n_{\tau-2d}+1}(i) - \bq_{n_{\tau-2d}}(i) }^+ } 
\\
& =
    \frac {2d+1} {d+1} \sum_{\tau = 2d + 1}^{T+d} \sum_{i\in[K]} \sum_{\sigma=\tau-2d}^{\tau-1} \sum_{n \in \N} \E\lsb{ \I\{\sigma \in \cU\} \I\{n_{\tau-2d} = n\} \lrb{ \bq_{n+1}(i) - \bq_{n}(i) }^+ }  
\\
& \overset{(\clubsuit)}{=}
    \frac {2d+1} {d+1} \sum_{\tau = 2d + 1}^{T+d} \sum_{\sigma=\tau-2d}^{\tau-1} \sum_{n \in \N} \E\bsb{ \I\{\sigma \in \cU\} \I\{n_{\tau-2d} = n\}  }\, \E\lsb{ \sum_{i\in[K]} \lrb{ \bq_{n+1}(i) - \bq_{n}(i) }^+ } 
\\
& \overset{(\spadesuit)}{\le}
    \frac {2d+1} {d+1} \xi \sum_{\tau = 2d + 1}^{T+d} \sum_{\sigma=\tau-2d}^{\tau-1} \sum_{n \in \N} \E\bsb{ \I\{\sigma \in \cU\} \I\{n_{\tau-2d} = n\}  } 
\\
& =
    \frac {2d+1} {d+1}\, \xi \sum_{\tau = 2d + 1}^{T+d} \sum_{\sigma=\tau-2d}^{\tau-1} \Pb[\sigma \in \cU]
\le
    \frac {2d+1} {d+1} 2d \beta(1-\beta)^{2d} \xi T
\end{align*}
where $(\clubsuit)$ follows by the independence of the Bernoulli sequence $\bb$ and the randomization of base algorithm $\alpha$;
$(\spadesuit)$ by the $\xi$-stability of $\alpha$;
and the last inequality by the fact that the probability that a time step $\sigma$ is an update round is $0$ if $\sigma \le 2d$ and $\beta (1-\beta)^{2d}$ otherwise.

For the second term $\due$, set for brevity 
\[
r_\tau \colon [K] \to \R~,\qquad i\mapsto \frac 1 {2d+1} \sum_{t=\tau -d }^\tau \brb{ \lcirc_t(i,\dots,i) - \lcirc_t (\is,\dots,\is) }
\]
for all $\tau \ge 2d+1$. We first note that
\begin{align*}
    \E \lsb{ \sum_{\tau \in \cU}  r_\tau(I_\tau) }
& =
    \E \lsb{ \sum_{\tau \in \cU}  r_\tau(I_{\tau-2d}) }
=
    \E \lsb{ \sum_{\tau = 2d+1}^T  r_\tau(I_{\tau-2d}) B_{\tau} \prod_{s=1}^{2d} (1-B_{\tau-s}) }
\\
& =
    \beta(1-\beta)^{2d}\sum_{\tau=2d + 1}^T \E \bsb{r_\tau (I_{\tau-2d})}
\end{align*}
where the first identity is a consequence of the fact that if $\tau\in \cU$, then the action $I_\tau$ played by \Cref{a:wrapper} at round $\tau$ coincides with the actions played in the previous $2d$ rounds, i.e., $I_{\tau}=\dots=I_{\tau-2d}$, while the last equality follows by the independence of $I_{\tau - 2d}$ and the vector of Bernoulli random variables $(B_{\tau}, \dots, B_{\tau-2d})$.
Thus,
\begin{align}
    \due
& = 
    \E \lsb{ \sum_{\tau = 2d+1}^T r_\tau (I_{\tau-2d}) }
=
    \frac{1}{\beta (1-\beta)^{2d}} \E \lsb{ \sum_{\tau \in \cU}  r_\tau(I_\tau) } \nonumber
\\
& =
    \frac{1}{\beta (1-\beta)^{2d}} \E \lsb{ \sum_{\tau \in \cU}  \frac 1 {2d+1} \sum_{t=\tau -d}^\tau \lcirc_t(I_\tau,\dots,I_\tau) - \sum_{\tau \in \cU}  \frac 1 {2d+1} \sum_{t=\tau -d}^\tau \lcirc_t (\is,\dots,\is) }
    \label{e:proof-update-term}
\end{align}
Now define for any $\tau\ge 2d + 1$,
\[
    \tilde \ell_\tau \colon [K] \to [0,1]\;, \quad
    i 
    \mapsto 
    \frac{1}{2d+1} \sum_{t = \tau-d}^\tau \lcirc_t (i,\dots,i)\;.
\]
(Note that $\tilde \ell_\tau(i) \in [0,1]$ for all $i\in [K]$ by Lemma~\ref{l:lemmino}.)
Leveraging again the fact that $\tau\in \cU$ implies $I_{\tau}=\dots=I_{\tau-2d}$, yields, for each $\tau \in \cU$,
\[
    \frac{1}{2d+1} \sum_{t = \tau-d}^\tau \lcirc_t (I_{t-d},\dots,I_t)
=
    \frac{1}{2d+1} \sum_{t = \tau-d}^\tau \lcirc_t (I_\tau,\dots,I_\tau)
=
    \tilde \ell_\tau(I_\tau) 
    \;.
\]
This shows that the loss \Cref{a:wrapper} feeds the base algorithm $\alpha$ (in \cref{s:update}) at each update round $\tau$ is a bandit feedback (for $\alpha$) for the arm $I_\tau$ with respect to the $[0,1]$-valued loss $\tilde \ell_\tau$. 
Therefore, by \Cref{e:proof-update-term} and the regret guarantees of the base algorithm, we obtain
\[
    \due
=   
    \frac{1}{\beta (1-\beta)^{2d}} \E \lsb{ \sum_{\tau \in \cU}  \tilde \ell_\tau (I_\tau) - \sum_{\tau \in \cU} \tilde \ell_\tau (\is)  }
\le
    \frac{1}{\beta (1-\beta)^{2d}} \cR_{\lfl{ \fracc{T}{(2d+1)
    } }}
\]
where in the last inequality we also used the monotonicity of the worst-case regret $\tau \mapsto \cR_\tau$ and the fact that $|\cU| \le \lce{ \frac{T - (2d+1) +1}{2d+1} } \le \lfl{\frac{T}{2d+1}}$.

In conclusion, we have
\begin{align*}
    R_T 
&\le 
    2d + 3\frac{2d+1}{d+1}d + \uno + \frac{2d+1}{d+1} \cdot \due 
\\
&\le 
    2d + \frac{2d+1}{d+1} \lrb{ 3d + 2d\beta(1-\beta)^{2d} \xi T + \frac{1}{\beta(1-\beta)^{2d}} \cR_{\lfl{ T/(2d+1) }} }
\end{align*}
hence concluding the proof.
\end{proof}

We can derive corollaries for various algorithms using \Cref{t:regret-CoLoWr}.
Consider for instance as a base algorithm $\alpha$ the well-known Exp3 algorithm of \cite{AuerCeFrSc02}.
It can be easily shown that Exp3 is $\eta$-stable (where $\eta$ is its learning rate, see Lemma~\ref{l:stabExp3} in Appendix~\ref{s:stabExp3}).
Combining this fact with its worst-case regret bound $\cR_N \le \frac{\ln K}{\eta} + \frac{\eta}{2} K N$, we obtain that the horizon-$T$ regret of the \colowr{} algorithm with $\alpha= \text{Exp3}(\eta)$, $\eta = \sqrt{\frac{d \ln K}{KT}}$, and an i.i.d.\ sequence $\bb$ of Bernoulli random variables with bias $\beta = \frac 1 {2d+1}$ (independent of the randomization of $\alpha$), satisfies $R_T = \cO \brb{ \sqrt{(d+1)KT\log K} }$.

Using Follow the Regularized Leader (FTRL) with $\frac 1 2$-Tsallis Entropy (\Cref{a:jake}), we can remove the $\log K$ term in the above bound.\footnote{The $\argmin$ where \Cref{a:jake} picks each distribution $\bq_n$ is always a singleton if $\eta>0$. This is a consequence of Lemma~\ref{l:lambda}, in Appendix~\ref{s:stab}. For the sake of convenience, we also allow the learning rate $\eta =0$, corresponding to a (non-regularized) Follow The Leader algorithm. In this case, multiple minimizers could exist but, for the sake of our results, ties could be broken in any (measurable) way.}
\begin{algorithm2e}
\DontPrintSemicolon
\SetKwInput{KwIn}{input}
\SetKwInput{kwInit}{initialize}
\KwIn{learning rate $\eta \ge 0$}
\kwInit{$\Lhat_0 = 0$}
    \For
    {
        round $n=1,2,\dots$
    }
    {
        Play action $J_n$ drawn according to
        \vspace{-1.5ex}
        \[
            \bq_n 
        \in 
            \argmin_{\bq \in \Delta_K} \lrb{ \sum_{i\in[K]} \Lhat_{n-1}(i) \bq(i) - 2 \eta \sum_{i\in[K]} \sqrt{\bq(i)} }\vspace{-2ex}
        \]
        where $\Delta_K$ is the probability simplex in $\R^K$\;
        Observe loss $\ell_n(J_n)$ and update $\Lhat_n = \Lhat_{n-1} + \lhat_n$, where
        \vspace{-1.5ex}
        \[
            \lhat_n(i)
        =
            \frac{\ell_n(i)}{\bq_n(i)} \I\{ J_n = i \}
        \qquad
        \forall i \in [K]
        \]\vspace{-4ex}\;
    }
\caption{\label{a:jake} Follow The Regularized Leader (FTRL) with $(\nicefrac 1 2)$-Tsallis entropy}
\end{algorithm2e}

However, we still need to prove the stability of \Cref{a:jake}. This is established by the following result, whose (non-trivial) proof is given in Appendix~\ref{s:stab}.
\begin{theorem}
\label{t:stability}
\Cref{a:jake} run with any learning rate $\eta>0$ is $\xi$-stable, with
\[
    \xi
\le
    2\frac{ 1 + \ln K }{\eta}
\]
\end{theorem}
The worst-case regret guarantees of FTRL with $\frac 1 2$-Tsallis entropy are as follows.\footnote{The analysis of \cite{abernethy2015fighting} is presented for a more general class of algorithms. A straightforward application of their Corollary~3.2 shows the validity of \Cref{t:jake} for \Cref{a:jake}.}

\begin{theorem}[\citealt{abernethy2015fighting}]
\label{t:jake}
For each $N \in \N$, the worst-case regret $\cR_N$ after $N$ rounds of \Cref{a:jake} run with learning rate $\eta = \sqrt{\frac{N}{2}}$ (in an adversarial setting with $[0,1]$-valued losses) satisfies
\[
    \cR_N
\le
    2 \sqrt { 2 K N }
\]
\end{theorem}
Combining \Cref{t:regret-CoLoWr,t:jake,t:stability}, we obtain the following regret bound for composite losses.

\begin{corollary}
\label{c:main-result}
For any time horizon $T\in \N$, if we run \colowr{} using:
\begin{itemize}[topsep=0pt,parsep=0pt,itemsep=0pt]
    \item As $\alpha$, \Cref{a:jake} with learning rate $\eta = \sqrt{\frac 1 2 \lfl{\fracc T {(2d+1)}}}$
    \item As $\bb$, an i.i.d.\ sequence of Bernoulli random variables with bias $\beta = \frac 1 {2d+1}$ (independent of the randomization of $\alpha$) 
\end{itemize}
then, its regret satisfies
\[
    R_T
\le
    c \sqrt{(d+1)KT}
\]
where $c = 2\sqrt{2}$ if $d=0$, and $c = 28$ otherwise.
\end{corollary}

\begin{proof}
If $d=0$, then $\beta = 1$ and \Cref{a:wrapper} reduces to the base algorithm.
The result in this case is therefore implied immediately by \Cref{t:jake}.

For the second part, we can assume without loss of generality that $T\ge 2d+1$, so that $\eta > 0$.
Then, plugging $\xi \le 2\frac{ 1 + \ln K }{\eta}$ (\Cref{t:stability}) and $\cR_{\lfl{ T/(2d+1) }} \le 2\sqrt{2K \lfl{ T/(2d+1) } }$ (\Cref{t:jake}) into \Cref{t:regret-CoLoWr} gives the result.
\end{proof}

\section{Lower bound}
\label{s:lower}
In this section we derive a lower bound for bandits with composite anonymous feedback. We do that through a reduction from the setting of linear bandits (in the probability simplex) to our setting. This reduction allows us to upper bound the regret of a linear bandit algorithm in terms of (a suitably scaled version of) the regret of an algorithm in our setting. Since the reduction applies to any instance of a linear bandit problem, we can use a known lower bound for the linear bandit setting to derive a corresponding lower bound for our composite setting.

Let $\Delta_K$ be the probability simplex in $\R^K$. At each round $t$, an algorithm $A$ for linear bandit optimization chooses an action $\bq_t\in\Delta_K$ and suffers loss $\bloss_t^{\top}\bq_t$, where $\bloss_t \in [0,1]^K$ is some unknown loss vector. The feedback observed by the algorithm at the end of round $t$ is the scalar $\bloss_t^{\top}\bq_t$. The regret suffered by algorithm $A$ playing actions $\bq_1,\dots,\bq_T$ is
\begin{equation}
\label{eq:lin-regret}
	\Rlin_T = \sum_{t=1}^T \bloss_t^{\top}\bq_t - \min_{\bq\in\Delta_K} \sum_{t=1}^T \bloss_t^{\top}\bq = \sum_{t=1}^T \bloss_t^{\top}\bq_t - \min_{i=1,\dots,K} \sum_{t=1}^T \bloss^\top \be_i
\end{equation}
where $\be_1,\dots,\be_K$ are the elements of the canonical basis of $\R^K$ and we used the fact that a linear function on the simplex is minimized at one of the corners.
Let $\cRlin_T(A,\Delta_K)$ denote the worst case regret (over the oblivious choice of $\bloss_1,\dots,\bloss_T$) of algorithm $A$. Similarly, let $\cR_T(A_d,K,d)$ be the worst case regret (over the oblivious choice of loss components $\loss_t^{(s)}(i)$ for all $t$, $s$, and $i$) of algorithm $A_d$ for nonstochastic $K$-armed bandits with $d$-delayed composite anonymous feedback. 
For the sake of clarity, we assume below that the time horizon $T$ is a multiple of $d+1$. If this is not the case, we can straightforwardly stop at the highest multiple of $d+1$ (smaller than $T$) up to paying an additive $\cO(d+1)$ regret.
Our reduction shows the following.
\begin{lemma}\label{lem:lower}
For any algorithm $A_d$ for $K$-armed bandits with $d$-delayed composite anonymous feedback, there exists an algorithm $A$ for linear bandits in $\Delta_K$ such that
$
	\cR_{T}(A_d,K,d) \ge (d+1)\,\cRlin_{T/(d+1)}(A,\Delta_K)
$.
\end{lemma}
Our reduction, described in detail in the proof of the above lemma (see Appendix~\ref{s:lower-appe}), essentially builds the probability vectors $\bq_t$ played by $A$ based on the empirical distribution of actions played by $A_d$ during blocks of size $d+1$. Now, an additional lemma is needed (whose proof is also given in the Appendix~\ref{s:lower-appe}).
\begin{lemma}
\label{l:shamir}
The regret of any algorithm $A$ for linear bandits on $\Delta_K$ satisfies $\cRlin_T(A,\Delta_K) = \widetilde{\Omega}\big(\sqrt{KT}\big)$.
\end{lemma}
In the previous lemma, as well as the following result, the $\widetilde{\Omega}$ notation is only hiding a $\sqrt{\log T}$ denominator.
Using the two lemmas above we can finally prove the lower bound.
\begin{theorem}
\label{t:lower}
For any algorithm $A_d$ for $K$-armed bandits with $d$-delayed composite anonymous feedback,
$
\cR_T(A_d,K,d)=\widetilde{\Omega}\big(\sqrt{(d+1)KT}\big)
$.
\end{theorem}
\begin{proof}
Fix an algorithm $A_d$. Using the reduction of~Lemma~\ref{lem:lower} gives an algorithm $A$ such that
$
	\cR_T(A_d,K,d) \ge (d+1)\,\cRlin_{T/(d+1)}(A,\Delta_K) = \widetilde{\Omega}\big(\sqrt{(d+1)KT}\big)
$,
where we used Lemma~\ref{l:shamir} with horizon $T/(d+1)$ to prove the $\widetilde{\Omega}$-equality.
\end{proof}
Although the loss sequence used to prove the lower bound for linear bandits in the simplex is stochastic i.i.d., the loss sequence achieving the lower bound in our delayed setting is not independent due to the deterministic loss transformation in the proof of Lemma~\ref{lem:lower} (which is defined independently of the algorithm, thus preserving the oblivious nature of the adversary).

\section{Conclusions}
We have investigated the setting of $d$-delayed composite anonymous feedback as applied to nonstochastic bandits. Composite anonymous feedback lends itself to formalize scenarios where the actions performed by the online decision-maker produce long-lasting effects that combine additively over time. A general reduction technique was introduced that enables the conversion of a stable algorithm working in a standard bandit framework into one working in the composite feedback framework. 
Applying our reduction to the FTRL algorithm with Tsallis entropy, we obtain an upper bounded on the regret of order $\sqrt{(d+1)KT}$.
We then showed the optimality of this rate (up to a logarithmic factor) relying on a lower bound for bandit linear optimization in the probability simplex.


\acks{%
NCB and RC were partially supported by the MIUR PRIN grant Algorithms, Games, and Digital Markets (ALGADIMAR) and the EU Horizon 2020 ICT-48 research and innovation action under grant agreement 951847, project ELISE (European Learning and Intelligent Systems Excellence).
The work of TC has benefited from the AI Interdisciplinary Institute ANITI, which is funded by the French ``Investing for the Future – PIA3'' program under the Grant agreement ANR-19-P3IA-0004. TC also acknowledges the support of the project BOLD from the French national research agency (ANR), and that of IBM.
YM was supported in part by the European Research Council (ERC) under the European Union's Horizon 2020 research and innovation program (grant agreement No. 882396), by the Israel Science Foundation (grant number 993/17), Tel Aviv University Center for AI and Data Science (TAD), and the Yandex Initiative for Machine Learning at Tel Aviv University.
Finally, we would like to thank 
Mengxiao Zhang and Chen-Yu Wei for finding a mistake in the proof of the main result in the preliminary version of this work \citep{cesa2018nonstochastic}. See \Cref{s:issues} for more details on the corrections we made.
}


\newpage

\appendix
\section{An Accountants' Lemma}
\label{s:ragionieri}

The next elementary lemma can be proved straightforwardly by swapping the order of the sums (see \Cref{f:ragionieri}).
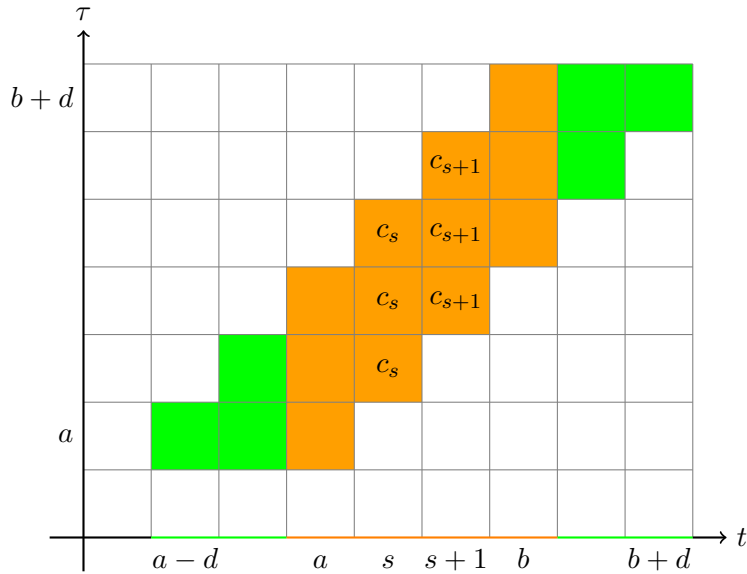
\begin{figure}
    \centering
    \begin{tikzpicture}[scale=0.9]
    \fill[green] (1,1) -- (3,1) -- (3,3) -- (2,3) -- (2,2) -- (1,2) -- cycle;
    \fill[orange!75!yellow] (3,1) -- (4,1) -- (4,2) -- (5,2) -- (5,3) -- (6,3) -- (6,4) -- (7,4) -- (7,7) -- (6,7) -- (6,6) -- (5,6) -- (5,5) -- (4,5) -- (4,4) -- (3,4) -- cycle;
    \fill[green] (7,5) -- (8,5) -- (8,6) -- (9,6) -- (9,7) -- (7,7) -- cycle;
    \foreach \i in {1,...,9}
    {
        \draw[gray,thin] (\i, 0) -- (\i, 7);
    }
    \foreach \i in {1,...,7}
    {
        \draw[gray,thin] (0, \i) -- (9, \i);
    }
    \draw[thick] (-0.5,0) -- (1,0);
        \draw[thick,green] (1,0) -- (3,0);
        \draw[thick,green] (7,0) -- (9,0);
        \draw[->,thick] (9,0) -- (9.5,0);
        \draw[thick,orange] (3,0) -- (7,0);
    \draw[->,thick] (0,-0.5) -- (0,7.5);
    \draw (9.5,0) node[right] {$t$};
    \draw (0,7.5) node[above] {$\tau$};
    \draw (0,1.5) node[left] {$a$};
    \draw (0,6.5) node[left] {$b+d$};
    \draw (1.5,0) node[below] {$a-d$};
    \draw (3.5,0) node[below] {$\vphantom{b}a$};
    \draw (4.5,0) node[below] {$\vphantom{b}s$};
    \draw (5.5,0) node[below] {$\vphantom{b}s+1$};    
    \draw (6.5,0) node[below] {$b$};
    \draw (8.5,0) node[below] {$b+d$};
    \foreach \i in {2.5,3.5,4.5}
        \draw (4.5,\i) node {$c_s$};
    \foreach \i in {3.5,4.5,5.5}
        \draw (5.5,\i) node {$c_{s+1}$};
    \end{tikzpicture}
    \caption{On the right-hand side of the equation in Lemma~\ref{l:ragionieri}, we are summing each row of the squares in the picture. On the left hand side, we are summing the columns. Each column $s$ contains the same constant $c_s$ in each component.}
    \label{f:ragionieri}
\end{figure}

\begin{lemma}
\label{l:ragionieri}
If $(c_t)_{t\in\Z} \subset \R$, $a,b \in \Z$ are such that $a \le b$ and $d \ge 0$ then
\[
    \sum_{t=a-d}^{a-1} (t-a+d+1)c_t + (d+1) \sum_{t=a}^b c_t + \sum_{t=b+1}^{b+d} (b+d+1-t)c_t 
= 
    \sum_{\tau = a}^{b+d} \sum_{t = \tau-d}^\tau c_t \;.
\]
\end{lemma}

\section{Stability of FTRL with Tsallis entropy}
\label{s:stab}

In this section, we prove a key stability property of FTRL with Tsallis entropy that could be of independent interest.
A general technique \cite[Lemma 2.10]{shalev2012online} to do so for the FTRL family of algorithms is to show that the regularizers are $\mu$-strongly convex with respect to the desired norm (in our case, the $\ell^1$-norm).
To the best of our knowledge, the existing results in this direction (e.g., \citealt[Section 7.3]{response}) lead to a (probably loose) upper bound on $1/\mu$ of order $K$, which in turn yields a suboptimal dependence on $K$ in the stability of FTRL with Tsallis entropy. 
To obtain the correct dependence on $K$ in the regret in Theorem 4, stability of order $\sqrt K$ or better is required instead.
If one wanted to follow this path, it would therefore be required to prove the tighter upper bound $1/\mu \lesssim \sqrt K$, which seems non-trivial.
Instead, we take a different route skipping this middle step and controlling the stability of FTRL with Tsallis entropy directly.

We begin by showing that, when $\eta>0$, there exists a unique solution of the optimization problem that defines $\bq_n$ in \Cref{a:jake} and provide a formula for it in terms of the corresponding Lagrange multiplier $\lambda$.
An analogous result was stated in \cite[Section~3.3]{zimmert2021tsallis} for the related algorithm Tsallis-INF.
\begin{lemma}
\label{l:lambda}
Let $\bw \in \R^K$ and $\eta >0$. Then 
\[
    \exists! \lambda_{\bw,\eta} < \min_{i\in [K]}\bw(i) 
    \;, \quad 
    \sum_{i\in[K]}\frac{\eta^2}{\brb{ \bw(i) - \lambda_{\bw,\eta} }^2} = 1 \;.
\]
Furthermore, defining for all $i\in[K]$,
\[
    \bq^{\mathrm{opt}}_{\bw,\eta}(i)
\coloneqq
    \frac{\eta^2}{\brb{ \bw(i)-\lambda_{\bw,\eta} }^2} \;,
\]
we have that $\bq^{\mathrm{opt}}_{\bw,\eta}$ is the unique global minimizer of the function:
\[
    \Delta_K \to \R
    \;, \quad
    \bq \mapsto \sum_{i\in[K]} \bw(i)\bq(i) - 2 \eta \sum_{i\in[K]} \sqrt{\bq(i)} \;.
\]
\end{lemma}
\begin{proof}
Define the two auxiliary functions
\begin{align*}
    f_{\bw,\eta} \colon [0,\iop)^K \to \R \;, \qquad
&
    \bq \mapsto \sum_{i\in [K]} \bw(i) \bq(i) - 2 \eta \sum_{i\in [K] } \sqrt{\bq(i)}
\\
    \fhi \colon [0,\iop)^K \to \R\;, \qquad
&
    \bq \mapsto \sum_{i\in [K]} \bq(i)
\end{align*}
Note that $f_{\bw,\eta}$ is strictly convex and continuous on $[0,\iop)^K$ and differentiable on $(0,\iop)^K$.
Thus, by the Lagrange multiplier theorem, a point $\bq \in (0,\iop)^K$ is the unique global minimizer of $f_{\bw,\eta}$ on the simplex $\Delta_K$ if and only if 
\begin{equation}
    \label{e:lagra}
    \fhi(\bq) = 1
\qquad \text{ and } \qquad
    \exists \lambda \in \R, \
    \nabla f_{\bw,\eta}(\bq) = \lambda \nabla \fhi(\bq)
\end{equation}
A direct verification shows that a pair $(\bq,\lambda) \in (0,\iop)^K \times \R$ satisfies condition~\eqref{e:lagra} if and only if
\begin{equation}
    \label{e:equiv-min}
    \lambda < \min_{i\in [K]} \bw(i) \;,
\quad 
    \sum_{i\in[K]} \frac{ \eta ^2}{\brb{ \bw(i) - \lambda }^2} = 1 \;,
\quad \text{ and } \quad
    \forall i \in [K]\;, \ \bq(i) = \frac{ \eta ^2}{\brb{ \bw(i) - \lambda }^2}
\end{equation}
Note that, letting $m \coloneqq \min_{i\in [K]} \bw(i)$, the function
\[
    g\colon ( -\iop, m ) \to \R \;, \quad \lambda \mapsto \sum_{i\in[K]} \frac{ \eta ^2}{\brb{ \bw(i) - \lambda }^2}
\]
is continuous, strictly increasing, and it satisfies
\[
    \lim_{\lambda \to -\iop} g(\lambda) = 0
\qquad \text{ and } \qquad
    \lim_{\lambda \to m^-} g(\lambda) = +\iop
\]
hence, there exists a unique $\lambda_{\bw,\eta} \in (-\iop,m)$ such that $g(\lambda_{\bw,\eta})=1$.
Therefore, $\bq^{\mathrm{opt}}_{\bw,\eta}$ is the unique global minimizer of $f_{\bw,\eta}$ on the simplex $\Delta_K$.
\end{proof}
This lemma controls the variation of the Lagrange multipliers corresponding to two points that vary by quantity $\delta$ only in a single coordinate.
\begin{lemma}
\label{l:diff_lambda}
Let $\bw\in\R^K$ such that $\bw(1)\le\dots\le\bw(K)$ and $\eta>0$.
Then, for all $\delta>0$ and $j\in[K]$, 
\[
    \lambda_{\bw+\delta\be_j,\eta} - \lambda_{\bw,\eta}
\le
    \frac 1 j \delta \;,
\]
where $\lambda_{\bw+\delta\be_j,\eta}$ and $\lambda_{\bw,\eta}$ are defined as in Lemma~\ref{l:lambda} and $\be_1,\dots,\be_K$ is the canonical basis of $\R^K$.
\end{lemma}
\begin{proof}
Define $\cX \coloneqq \bcb{ \brb{ \bu(1), \dots, \bu(K),\lambda } \in \R^K\times\R \mid \lambda < \min_{i\in [K]} \bu(i) }$ and
\begin{align*}
    \psi  \colon & \cX \to \R\;, \quad \brb{ \bu(1),\dots,\bu(K), \lambda } \mapsto \sum_{i\in[K]} \frac{\eta^2}{\brb{ \bu(i) -\lambda }^2}
\\
    \Lambda \colon & \R^K \to \R\;, \quad \bu \mapsto \lambda_{\bu, \eta}
\end{align*}
where $\lambda_{\bu,\eta}$ is defined as in Lemma~\ref{l:lambda}.
By the implicit function theorem, we have that $\Lambda$ is infinitely differentiable and for all $\bu \in \R^K$, 
\begin{align*}
    D_j \Lambda (\bu)
&
=
    -\frac{D_j \psi \brb{\bu(1),\dots,\bu(K), \lambda_{\bu,\eta}} }{D_{K+1} \psi \brb{\bu(1),\dots,\bu(K), \lambda_{\bu,\eta}} }
\\
&
=
    -\frac{ -2\frac{ \eta^2 }{\lrb{ \bu(j) -\lambda_{\bu,\eta} }^3} }{ 2 \sum_{i\in[K]}\frac{ \eta^2 }{\lrb{ \bu(i) -\lambda_{\bu,\eta} }^3} }
=
    \frac 1 { 1 + \sum_{i\in[K], i\neq j} \lrb{ \frac{ \bu(j) - \lambda_{\bu,\eta} }{ \bu(i) - \lambda_{\bu,\eta}} }^3 } \;,
\end{align*}
where we denoted the partial derivative with respect to the $j$-th coordinate by $D_j$. 
Now, by the fundamental theorem of calculus,
\begin{align*}
&
    \lambda_{\bw+\delta\be_j,\eta} - \lambda_{\bw,\eta}
=
    \Lambda(\bw +\delta\be_j) - \Lambda(\bw)
=
    \delta\int_0^1 D_j \Lambda(\bw + s \delta \be_j) \dif s
\\
&
\quad
=
    \delta \int_0^1 \frac 1 { 
    1 + 
    \sum_{i\in[K], i \le j-1} \lrb{ \frac{ \bw(j) + s \delta - \lambda_{\bw + s \delta \be_j,\eta} }{ \bw(i) - \lambda_{\bw + s \delta \be_j,\eta}} }^3
    +
    \sum_{i\in[K], i \ge j+1} \lrb{ \frac{ \bw(j) + s \delta - \lambda_{\bw + s \delta \be_j,\eta} }{ \bw(i) - \lambda_{\bw + s \delta \be_j,\eta}} }^3 
    } \dif s
\\
&
\quad
\le
    \delta \int_0^1 \frac 1 { 
    1 + 
    \sum_{i\in[K], i \le j-1} \lrb{ \frac{ \bw(i) - \lambda_{\bw + s \delta \be_j,\eta} }{ \bw(i) - \lambda_{\bw + s \delta \be_j,\eta}} }^3
    } \dif s
=
    \frac 1 j \delta \;.
\end{align*}
\end{proof}
We can finally prove the stability of FTRL with $\frac 1 2$-Tsallis entropy.

\medskip

\begin{proof}\textbf{of \Cref{t:stability}}
Consider an arbitrary sequence of losses $(\ell_m)_{m\in \N} \s [0,1]$.
Fix any $\eta>0$. 
For each $\bw \in \R^K$, let $\lambda_{\bw} \coloneqq \lambda_{\bw,\eta}$, where $\lambda_{\bw,\eta}$ is defined as in Lemma~\ref{l:lambda}.
Let $\cF_0$ be the trivial $\sigma$-algebra (containing only the sample space and the empty set) and for all $n\in \N$, $\cF_n \coloneqq \sigma(J_1,\ldots J_n)$.
Note that, for each $n\in \N$, we have that $\Lhat_{n-1}$ is $\cF_{n-1}$-measurable and, as a consequence of Lemma~\ref{l:lambda}, for all $i\in[K]$,
\[
    \bq_n(i)
=
    \frac{ \eta^2 } { \brb{ \Lhat_{n-1}(i) - \lambda_{\Lhat_{n-1}} }^2 }\;.
\]
Let $\be_1,\dots,\be_K$ be the canonical basis of $\R^K$ and fix any $n\in \N$.
Define for each $j\in[K]$, $\lhat_{n,j} \coloneqq \frac{ \ell_n(j) }{\bq_n(j)} \be_j$ and note that $\lhat_n = \lhat_{n,J_n}$. To make the notation more compact, we also let $\E_n[\cdot] \coloneqq \E[\cdot \mid \cF_{n-1}]$.
Then
\begin{align*}
&
    \E_n \lsb{ \sum_{i\in[K]} \brb{ \bq_{n+1}(i) - \bq_n(i) }^+ }
=
    \E_n \lsb{ \sum_{i\in[K]} \lrb{ \frac{ \eta^2 } { \brb{ \Lhat_{n}(i) - \lambda_{\Lhat_{n}} }^2 } - \frac{ \eta^2 } { \brb{ \Lhat_{n-1}(i) - \lambda_{\Lhat_{n-1}} }^2 } }^+ }
\\
&
=
    \E_n \lsb{ \sum_{i\in[K]} \lrb{ \frac{ \eta^2 } { \brb{ \Lhat_{n-1}(i) + \lhat_{n,J_n}(i) - \lambda_{\Lhat_{n-1} + \lhat_{n,J_n}} }^2 } - \frac{ \eta^2 } { \brb{ \Lhat_{n-1}(i) - \lambda_{\Lhat_{n-1}} }^2 } }^+ }
\\
&
=
    \sum_{j\in[K]}\E_n \lsb{ \I\{J_n = j\} \sum_{i\in[K]} \lrb{ \frac{ \eta^2 } { \brb{ \Lhat_{n-1}(i) + \lhat_{n,j}(i) - \lambda_{\Lhat_{n-1} + \lhat_{n,j}} }^2 } - \frac{ \eta^2 } { \brb{ \Lhat_{n-1}(i) - \lambda_{\Lhat_{n-1}} }^2 } }^+ }
\\
&
=
    \sum_{j\in[K]}\E_n \bsb{ \I\{J_n = j\} } \sum_{i\in[K]} \lrb{ \frac{ \eta^2 } { \brb{ \Lhat_{n-1}(i) + \lhat_{n,j}(i) - \lambda_{\Lhat_{n-1} + \lhat_{n,j}} }^2 } - \frac{ \eta^2 } { \brb{ \Lhat_{n-1}(i) - \lambda_{\Lhat_{n-1}} }^2 } }^+
\\
&
=
    \sum_{j\in[K]} \bq_n(j) \sum_{i\in[K]} \lrb{ \frac{ \eta^2 } { \brb{ \Lhat_{n-1}(i) + \lhat_{n,j}(i) - \lambda_{\Lhat_{n-1} + \lhat_{n,j}} }^2 } - \frac{ \eta^2 } { \brb{ \Lhat_{n-1}(i) - \lambda_{\Lhat_{n-1}} }^2 } }^+
=
    (\star)\;.
\end{align*}
Now, note that, for each $j \in [K]$:
\begin{itemize}
    \item $\lambda_{\Lhat_{n-1} + \lhat_{n,j}} \ge \lambda_{\Lhat_{n-1}}$ (because, as we show in the proof of Lemma~\ref{l:diff_lambda}, the directional derivatives of $\bu \mapsto \lambda_{\bu}$ are positive).
    \item For each $i \in [K] \backslash \{j\}$, 
    $
        \frac{\eta^2}{\Lhat_{n-1}(i) - \lambda_{\Lhat_{n-1}}}
    \le
        \frac{\eta^2}{\Lhat_{n-1}(i) - \lambda_{\Lhat_{n-1} + \lhat_{n,j}}}
    = 
        \frac{\eta^2}{\Lhat_{n-1}(i) + \lhat_{n,j}(i) - \lambda_{\Lhat_{n-1} + \lhat_{n,j}}} $ (by the previous point).
    \item 
    $
        \frac{\eta^2}{\Lhat_{n-1}(j) - \lambda_{\Lhat_{n-1}}} 
    \ge 
        \frac{\eta^2}{\Lhat_{n-1}(j) + \lhat_{n,j}(j) - \lambda_{\Lhat_{n-1} + \lhat_{n,j}}}
    $ (by the previous point and the fact that $\sum_{i\in[K]} \frac{\eta^2}{\Lhat_{n-1}(i) - \lambda_{\Lhat_{n-1}}} = 1 = \sum_{i\in[K]} \frac{\eta^2}{\Lhat_{n-1}(i) + \lhat_{n,j}(i) - \lambda_{\Lhat_{n-1} + \lhat_{n,j}}} $ ).
\end{itemize}
It follows that
\begin{align*}
&
    (\star)
=
    \sum_{j\in[K]} \bq_n(j) \sum_{i\in[K], i \neq j} \lrb{ \frac{ \eta^2 } { \brb{ \Lhat_{n-1}(i) - \lambda_{\Lhat_{n-1} + \lhat_{n,j}} }^2 } - \frac{ \eta^2 } { \brb{ \Lhat_{n-1}(i) - \lambda_{\Lhat_{n-1}} }^2 } }
\\
&
\hspace{8.92575pt}
=
    \eta^2 \sum_{j\in[K]} \bq_n(j) \sum_{i\in[K], i \neq j} \frac{ \brb{ \Lhat_{n-1}(i) - \lambda_{\Lhat_{n-1}} }^2 - \brb{ \Lhat_{n-1}(i) - \lambda_{\Lhat_{n-1} + \lhat_{n,j}} }^2 } { \brb{ \Lhat_{n-1}(i) - \lambda_{\Lhat_{n-1} + \lhat_{n,j}} }^2 \brb{ \Lhat_{n-1}(i) - \lambda_{\Lhat_{n-1}} }^2 } 
\\
&
\hspace{8.92575pt}
=
    \eta^2 \sum_{j\in[K]} \bq_n(j) (\lambda_{\Lhat_{n-1} + \lhat_{n,j}} - \lambda_{\Lhat_{n-1}}) \sum_{i\in[K], i \neq j}  \frac{ \brb{ \Lhat_{n-1}(i) - \lambda_{\Lhat_{n-1}} } + \brb{ \Lhat_{n-1}(i) - \lambda_{\Lhat_{n-1} + \lhat_{n,j}} } } { \brb{ \Lhat_{n-1}(i) - \lambda_{\Lhat_{n-1} + \lhat_{n,j}} }^2 \brb{ \Lhat_{n-1}(i) - \lambda_{\Lhat_{n-1}} }^2 }
\\
&
\hspace{8.92575pt}
\le
    2\eta^2 \sum_{j\in[K]} \bq_n(j) (\lambda_{\Lhat_{n-1} + \lhat_{n,j}} - \lambda_{\Lhat_{n-1}}) \sum_{i\in[K], i \neq j}  \frac{1} { \brb{ \Lhat_{n-1}(i) - \lambda_{\Lhat_{n-1} + \lhat_{n,j}} }^3 }
\\
&
\hspace{8.92575pt}
\le
    2\eta^2 \sum_{j\in[K]} \bq_n(j) (\lambda_{\Lhat_{n-1} + \lhat_{n,j}} - \lambda_{\Lhat_{n-1}}) \lrb{ \sum_{i\in[K], i \neq j} \frac{1} { \brb{ \Lhat_{n-1}(i) - \lambda_{\Lhat_{n-1} + \lhat_{n,j}} }^2 } }^{3/2}
\\
&
\hspace{8.92575pt}
=
    2\eta^2 \sum_{j\in[K]} \bq_n(j) (\lambda_{\Lhat_{n-1} + \lhat_{n,j}} - \lambda_{\Lhat_{n-1}}) \lrb{ \sum_{i\in[K], i \neq j} \frac{1} { \brb{ \Lhat_{n-1}(i) + \lhat_{n,j}(i) - \lambda_{\Lhat_{n-1} + \lhat_{n,j}} }^2 } }^{3/2}
\\
&
\hspace{8.92575pt}
\le
    2\eta^2 \sum_{j\in[K]} \bq_n(j) (\lambda_{\Lhat_{n-1} + \lhat_{n,j}} - \lambda_{\Lhat_{n-1}}) \lrb{ \sum_{i\in[K]} \frac{1} { \brb{ \Lhat_{n-1}(i) + \lhat_{n,j}(i) - \lambda_{\Lhat_{n-1} + \lhat_{n,j}} }^2 } }^{3/2}
\\
&
\hspace{8.92575pt}
=
    \frac{2}{\eta} \sum_{j\in[K]} \bq_n(j) (\lambda_{\Lhat_{n-1} + \lhat_{n,j}} - \lambda_{\Lhat_{n-1}}) = (\star\star) \;.
\end{align*}
Now, let $\sigma$ be a random permutation of $[K]$ such that $\Lhat_{n-1}\brb{\sigma(1)}\le \dots \le \Lhat_{n-1}\brb{\sigma(K)}$. Then, using Lemma~\ref{l:diff_lambda}, we have
\begin{align*}
    (\star\star)
&
=
    \frac{2}{\eta} \sum_{j\in[K]} \bq_n\brb{\sigma(j)} (\lambda_{\Lhat_{n-1} + \lhat_{n,\sigma(j)}} - \lambda_{\Lhat_{n-1}})
\\
& =
    \frac{2}{\eta} \sum_{j\in[K]} \bq_n\brb{\sigma(j)} \lrb{ \lambda_{\Lhat_{n-1} + \frac{ \ell_n (\sigma(j)) }{ \bq_n (\sigma(j)) } \be_{\sigma(j)} } - \lambda_{\Lhat_{n-1}} }
\\
&
\le
    \frac{2}{\eta} \sum_{j\in[K]} \bq_n\brb{\sigma(j)} \frac{1}{j} \frac{ \ell_n (\sigma(j)) }{ \bq_n (\sigma(j)) }
\le
    \frac{2}{\eta} \sum_{j\in[K]} \frac{1}{j}
\le
    2 \frac{1+\ln K}{\eta} \;.
\end{align*}
In conclusion:
\[
    \E \lsb{ \sum_{i\in[K]} \brb{ \bq_{n+1}(i) - \bq_n(i) }^+ \mid \cF_{n-1} }
\le
    2 \frac{1+\ln K}{\eta} \;.
\]
It follows that:
\[
    \E \lsb{ \sum_{i\in[K]} \brb{ \bq_{n+1}(i) - \bq_n(i) }^+ }
=
    \E \lsb{ \E \lsb{ \sum_{i\in[K]} \brb{ \bq_{n+1}(i) - \bq_n(i) }^+ \mid \cF_{n-1} } }
\le
    2 \frac{1+\ln K}{\eta}  \;.
\]
Being $n$ and $(\ell_m)_{m\in \N}$ arbitrary, the result follows.
\end{proof}

\section{Stability of Exp3}
\label{s:stabExp3}

In this section, we prove the stability of Exp3.

\begin{lemma}
\label{l:stabExp3}
Exp3 with learning rate $\eta$ is $\xi$-stable with $\xi=\eta$.
\end{lemma}
\begin{proof}
Consider an arbitrary sequence of losses $(\ell_n)_{n\in \N} \s [0,1]$.
In this case, stability holds pointwise (for all realizations of the actions $J_1,J_2,\dots$ played by Exp3 on the sequence of losses $(\ell_n)_{n\in \N}$) rather that in expectation. From~\cite[Lemma 1]{cgmm16} we have, for any round $n\in \N$ and all arms $i\in [K]$,
\[
	\bq_{n+1}(i) - \bq_n(i) \leq \eta\,\bq_{n+1}(i) \sum_{j=1}^K \bq_n(j) \lhat_n(j)
\]
where $\lhat_n(j)=\frac{\ell_n(j) \I\{J_n=j\}}{\bq_n(j)}$, $\bq_n(j)=\bw_n(j)/\sum_{k=1}^K \bw_n(k)$, and if $n=1$, $\bw_n(k)=1$ while, if $n \ge 2$, $\bw_n(k)$ is defined inductively by $\bw_{n}(k) = \bq_{n-1}(k) e^{-\eta \lhat_{n-1}(k)}$.
Hence we can write, for any $n\in \N$,
\begin{align*}
\sum_{i\,:\,\bq_{n+1}(i) > \bq_n(i)} \brb{ \bq_{n+1}(i)-\bq_n(i) }
&\leq
\sum_{i\,:\,\bq_{n+1}(i) > \bq_n(i)} \eta\,\bq_{n+1}(i) \sum_{j=1}^K \bq_n(j) \lhat_n(j)\\
&=
\sum_{i\,:\,\bq_{n+1}(i) > \bq_n(i)} \eta\,\bq_{n+1}(i) \ell_n(J_n)
\leq
\eta\,\sum_{i\,:\,\bq_{n+1}(i) > \bq_n(i)} \bq_{n+1}(i)
\leq
\eta \;.
\end{align*}
Being $(\ell_n)_{n\in \N}$ arbitrary, the result follows.
\end{proof}

\section{Lower Bound (missing proofs)}
\label{s:lower-appe}

We begin this section by showing a reduction mapping each algorithm for bandits with composite anonymous feedback to one for linear bandits with a better or equal regret.

\medskip

\begin{proof}\textbf{of Lemma \ref{lem:lower}}
Fix an instance $\bloss_1,\dots,\bloss_{T/(d+1)}$ of a linear bandit problem and use it to construct an instance of the $d$-delayed bandit setting with loss components
\[
    \loss_t^{(s)}(i) = \left\{ \begin{array}{cl}
        \bloss_{\lceil t/(d+1)\rceil}^\top\be_i & \text{if $t+s = 0 \quad (\mathrm{mod}\ d+1)$,}
    \\[1mm]
        0 & \text{otherwise,}
    \end{array} \right.
\]
where $\be_1,\dots,\be_K$ are the elements of the canonical basis of $\R^K$.
These components define the following composite loss incurred by any algorithm $A_d$ playing actions $I_1,I_2,\dots$
\[
    \lcomp_t(I_{t-d},\dots,I_t)
=
    \sum_{s=0}^{d} \loss_{t-s}^{(s)}(I_{t-s})
=
    \left\{ \begin{array}{cl}
        (d+1)\,\bloss_{\lceil t/(d+1)\rceil}^\top \bq_t & \text{if $t =0 \quad (\mathrm{mod}\ d+1)$,}
    \\[1mm]
        0 & \text{otherwise,}
    \end{array} \right.
\]
where $\bq_t$ is defined from $I_{t-d},\dots,I_t \in [K]$ as follows
\begin{equation}
\label{eq:p-reduction}
    \bq_t(j) = \frac{1}{d+1}\sum_{s=t-d}^t \I\{I_s = j\} \qquad j\in[K].
\end{equation}
Note that $\bq_t(i)$ is the fraction of times action $i$ was played by $A_d$ in the last $d+1$ rounds.
Given the algorithm $A_d$, we define the algorithm $A$ for playing linear bandits on the loss sequence $\bloss_1,\dots,\bloss_{T/(d+1)}$ as follows. If $t \neq 0 \  (\mathrm{mod}\ d+1)$, then $A$ skips the round. On the other hand, when $t=0 \ (\mathrm{mod}\ d+1)$, $A$ performs action $\bq_t$ defined in~(\ref{eq:p-reduction}), observes the loss $\bloss_{\lceil t/(d+1)\rceil}^\top\bq_t$, and returns to $A_d$ the composite loss $\lcomp_t(I_{t-d},\dots,I_t)$. Essentially, $A_d$ observes a nonzero composite loss only every $d$ time steps, when $t =0 \ (\mathrm{mod}\ d+1)$. When this happens, the composite loss of $A_d$ is $d\,\bloss_{\lceil t/(d+1)\rceil}^\top\bq_t$, which is $d+1$ times the loss of $A$.

Now it is enough to note that, using~(\ref{eq:lin-regret}),
\[
	\min_{k=1,\dots,K} \sum_{t=1}^T \lcomp_t(k,\dots,k)
=
	\min_{k=1,\dots,K} (d+1) \sum_{s=1}^{T/(d+1)} \bloss_s^\top \be_k
=
	\min_{\bq\in\Delta_K} (d+1) \sum_{s=1}^{T/(d+1)} \bloss_s^\top \bq~.
\]
This concludes the proof.
\end{proof}
We remark that our lower bound construction relies crucially on the power of the adversary to \emph{plan} the assignment of delays: losses of order $d$ are revealed only on $T/d$ time steps, leading to a multiplicative dependence on $d$.
This stacking effect is not possible in settings like the ones studied in \cite{pike2018bandits}, where delays are drawn i.i.d.\ over rounds and are, therefore, \emph{independently spread} across time steps.

We now prove a lower bound for linear bandits.

\medskip

\begin{proof}\textbf{of Lemma~\ref{l:shamir}}
The statement is essentially proven in \citep[Theorem~5]{shamir2015complexity}, where the author shows a $\Omega\big(\sqrt{K/T}\big)$ lower bound on the error of \textsl{bandit linear optimization} in the probability simplex.\footnote
{%
It is worth stressing that the lower bound in \citep{shamir2015complexity} is based on stochastic i.i.d.\ generation of losses, hence it does not violate our assumption about the obliviousness of the adversary.
}
As explained in \citep[Section~1.1]{shamir2015complexity}, (cumulative) regret lower bounds for linear bandits can be obtained by multiplying the lower bounds on bandit linear optimization error by $T$. A possible issue is that the proof in \citep[Theorem~5]{shamir2015complexity} uses unbounded Gaussian losses. However, in \citep[Appendix~B]{shamir2015complexity} it is shown how lower bounds for Gaussian losses can be converted into lower bounds for losses in $[-1,1]$ at the cost of a $1\big/\sqrt{\ln T}$ factor in the regret. Finally, note that our setting requires losses in $[0,1]$, but this is not an issue either because we are in a linear setting, and thus we can add the $(1,\dots,1)$ constant vector to all loss vectors without affecting the regret.
\end{proof}

\section{Issues in the Preliminary Version of the Paper}
\label{s:issues}

This paper is an extended and improved version of a preliminary paper appeared as \citep{cesa2018nonstochastic}.
There are two main issues in the preliminary version. 

Firstly, in \citet[last line of Eq. (11)]{cesa2018nonstochastic}, we find the inequality
        \[
            \E \lsb{ \sum_{s=0}^{d-1} \sum_{i:p_{t-s}(i) > p_{t-d+1}(i)} \lrb{ p_{t-s}(i) - p_{t-d+1}(i) } }
        \le
            \xi
        \]
    but, whenever an update occurred at round $t-d+2$, the same term is summed $\Theta(d)$ times (rather than $1$), leading to a bound of order $\Theta(d\xi)$ rather than the claimed $\Theta(\xi)$, and a consequently looser upper bound for the performance of the wrapper.
    
Secondly, in \citet[first line of Eq. (10)]{cesa2018nonstochastic}, we find the inequality
        \[
            \E \lsb{ \sum_{t\in \cU, t\ge 2d-2} \Delta_t^k }
        \ge
            q\brb{ 1 - q(2d -1) } \sum_{t=2d-2}^T \E[\Delta_t^k]
        \]
    which would follow from the provided discussion on $\Pb'\lrb{ \bigwedge_{s=1}^{2d-1} (t-s\notin\cU) }$ \emph{if $\Delta_t^k$ were non-negative}, but this is not necessarily the case.

The present version proposes a different wrapper and patch both things up, as described below.

When analyzing the term corresponding to the first issue in \citep{cesa2018nonstochastic}, we take a different route. 
We begin with a change of variables and never upper bound the losses $\ell_{t-s}^{(s)}$ with $1$, relying instead on Lemma~\ref{l:lemmino} to obtain the correct dependence on $d$.
This can be seen in the upper bound of $\uno$ in the proof of \Cref{t:regret-CoLoWr}.

For the second issue, the problem with the original wrapper in \citep{cesa2018nonstochastic} is to disentangle $\Delta_t^k$ from the random variable $\I\{t \text{ is an update round}\}$.
There is, however, an easier way to get around this roadblock, taking a slightly different route and relying on a different definition of draw, stay, and update rounds (\`a la \cite{dekel2014blinded}, see \Cref{d:update-rounds} and the subsequent discussion). 
This greatly simplifies the analysis as can be seen in the upper bound of $\due$ in the proof of \Cref{t:regret-CoLoWr}.

\vskip 0.2in
\bibliography{biblio}

\end{document}